\documentclass{article}
\usepackage{ijcai18}

\usepackage{times}
\usepackage{xcolor}
\usepackage{soul}
\usepackage[utf8]{inputenc}
\usepackage[small]{caption}

\usepackage{amssymb,amsmath,amsfonts,bm,epsfig,graphicx,theorem,epstopdf}
\usepackage{rotating,setspace,latexsym,epsf,color,dsfont,caption,mathtools}
\usepackage{subfigure}
\usepackage{algorithm}
\usepackage{algpseudocode}
\usepackage{bbm}
\usepackage{hyperref}
\usepackage{booktabs}

\newtheorem{Theorem}{Theorem}

\newtheorem{Lemma}{Lemma}

\newtheorem{Definition}{Definition}
\newtheorem{Assumptions}{Assumption}
\newenvironment{Proof}[1]{\medskip\par\noindent
	{\bf Proof:\,}\,#1}{{\mbox{\,$\blacksquare$}\par}}

\newenvironment{proof}[1]{\medskip\par\noindent
	{\bf Proof:\,}\,#1}{{\mbox{\,$\blacksquare$}\par}}

\newcommand{\Eb}{\mathbb{E}}
\newcommand{\Pb}{\mathbb{P}}

\newcommand{\thetav}{\boldsymbol{\theta}}

\newcommand{\Bc}{\mathcal{B}}
\newcommand{\Ec}{\mathcal{E}}
\newcommand{\Hc}{\mathcal{H}}

\ifodd 0
\newcommand{\congr}[1]{{\color{magenta}#1}}
\else
\newcommand{\congr}[1]{#1}
\fi

\ifodd 0
\newcommand{\congc}[1]{{\color{magenta}(Cong: #1)}}
\else
\newcommand{\congc}[1]{}
\fi

\ifodd 0
\newcommand{\jing}[1]{{\color{red}#1}}
\else
\newcommand{\jing}[1]{#1}
\fi

\ifodd 0
\newcommand{\ruida}[1]{{\color{blue}#1}}
\else
\newcommand{\ruida}[1]{#1}
\fi

\ifodd 1

\else
\newcommand{\ruida}[1]{#1}
\fi

\ifodd 0
\newcommand{\com}[1]{{\color{orange}(JY: #1)}}
\else
\newcommand{\com}[1]{}
\fi

\ifodd 0
\newcommand{\rev}[1]{{\color{red} #1}}
\else
\newcommand{\rev}[1]{}
\fi

\title{Cost-aware Cascading Bandits}

\author{
Ruida Zhou$^1$,
Chao Gan$^2$,
Jing Yang$^2$,
Cong Shen$^1$
\\
$^1$ University of Science and Technology of China \\
$^2$ The Pennsylvania State University\\
zrd127@mail.ustc.edu.cn,
cug203@psu.edu,
yangjing@psu.edu,
congshen@ustc.edu.cn
}

\date{}

\begin{document}
\maketitle

\begin{abstract}
In this paper, we propose a cost-aware cascading bandits model, a new variant of multi-armed bandits with cascading feedback, by considering the random cost of pulling arms. In each step, the learning agent chooses an {\it ordered} list of items and \congr{examines} them sequentially, until certain stopping condition is satisfied. Our objective is then to maximize the expected {\it net reward} in each step, i.e., the reward obtained in each step minus the total cost incurred in examining the items, by deciding the ordered list of items, as well as when to stop examination. We study both the offline and online settings, depending on whether the state and cost statistics of the items are known beforehand. For the offline setting, we show that the Unit Cost Ranking with Threshold 1 (UCR-T1) policy is optimal. For the online setting, we propose a Cost-aware Cascading Upper Confidence Bound (CC-UCB) algorithm, and show that the cumulative regret scales in $O(\log T)$. We also provide a lower bound for all $\alpha$-consistent policies, which scales in $\Omega(\log T)$ and matches our upper bound. The performance of the CC-UCB algorithm is evaluated with both synthetic and real-world data. 
\end{abstract}

\section{Introduction}
In this paper, we introduce a new cost-aware cascading bandits (CCB) model.
We consider a set of $K$ items (arms) denoted as $[K] = \{1,2,\ldots,K\}$. Each item $i\in[K]$ has two possible states 0 and 1, which evolve according to an independent and identically distributed (i.i.d.) Bernoulli random variable. The learning agent chooses an {\it ordered} list of items in each step and \congr{examines} them sequentially until certain stopping condition is met. The reward that the learning agent receives in a step equals one if one of the examined items in that step has state 1; Otherwise, it equals zero. We \congr{associate} a \congr{random} cost for examining each item. The overall reward function, termed as {\it net reward}, is the reward obtained in each step minus the total cost incurred in examining the items before the learner stops.

The CCB model \congr{is a natural but technically non-trivial extension of cascading bandits \cite{kveton2015cascading}, and is a more suitable model} in many fields, including the following examples.
\begin{itemize}
\item[1)] \emph{Opportunistic spectrum access.} In cognitive radio systems, a user is able to probe multiple channels sequentially at the beginning of a transmission session before it decides to use at most one of the channels for data transmission. Assuming the channel states evolve between {\it busy} and {\it idle} in time, the user gets a reward if there exists one {\it idle} channel for transmission. The cost then corresponds to the energy and delay incurred in probing each channel.

\item[2)] \emph{Dynamic treatment allocation.} In clinic trials, a doctor must assign one of several treatments to a patient in order to cure a disease. The doctor accrues information about the outcome of previous treatments before making the next assignment. Whether a treatment cures the disease can be modeled as a Bernoulli random variable, and the doctor gets a reward if the patient is cured. The doctor may not only be interested in the expected effect of the treatment but also its riskiness, which can be interpreted as the cost of the treatment.
\end{itemize}

In \congr{both examples}, the net reward in each step is determined not only by the subset of items included \congr{in} the list, but also by the \congr{{\it order}} that they are pulled. Intuitively, if the costs in examining the items are homogeneous, we would prefer to have the channel with higher probability to be idle, or the more effective treatment ranked higher \congr{in} the list. Then, the learner would find an available channel or cure the patient after a few attempts and then stop examination, thus saving the cost without decreasing the reward. However, for more general cases where the costs are heterogeneous \congr{or even random}, the \congr{optimal} solution is \congr{not immediately clear}.

In this paper, we consider a general cost model where the costs of pulling arms are heterogeneous and random, and investigate the corresponding solutions. 
Our main contributions are three-fold:

\begin{itemize}
\item[1)] We propose a novel CCB model, which has implications in many practical scenarios, such as opportunistic spectrum access, dynamic treatment allocation, etc. The CCB model is fundamentally different from its cost-oblivious counterparts, and admits \congr{a} unique structure in the corresponding learning strategy.

\item[2)] Second, with a priori statistical knowledge of the arm states and the costs, we explicitly identify the special structure of optimal policy (coined as \congr{the} \emph{offline} policy), which serves as the baseline for the online algorithm we develop. The optimal offline policy\congr{, called Unit Cost Ranking with Threshold 1 (UCR-T1),} is to {\it rank} the arms based the statistics of their states and costs, and pull those above certain threshold sequentially until a state 1 is observed.

\item[3)] Third, we propose a cost-aware cascading Upper Confidence Bound \congr{(CC-UCB)} algorithm for the scenario when prior arm statistics are unavailable, and show that \congr{it is} order-optimal by establishing order-matching upper and lower bounds on the regret. Our analysis indicates that the UCB based algorithm performs well for ranking the arms, i.e., the cumulative regret of ranking the desired arms in a wrong order is bounded.
\end{itemize}

\section{Related Literature}
\if{0}
As a simple yet expressive model for understanding the exploration-exploitation tradeoffs facing unknown environment, multi-armed bandit (MAB) and its variants have been extensively studied in the past decades~\cite{lai1985asymptotically,auer2002finite,bubeck2012regret,agrawal2012analysis,agrawal2013further}. \congc{We may consider deleting this paragraph if we need more space. I think this paragraph is very trivial for the reviewers.}

\congc{If we do delete the above paragraph, we may consider using subsections. It can be "Bandits with cost and budget constraint"; "Multiple-play bandits"; and "Ranked bandits".}
\fi
There have been some attempts that take the cost of pulling arms and budget constraint into the multi-armed bandit (MAB) framework recently. They can be summarized in two types. In the first type \cite{Audibert:COLT:2010,Bubeck:2009:PureExp,Guha:2007:STOC}, pulling each arm in the exploration phase has a unit cost and the goal is to find the best arm given the budget constraint on the total number of exploration arms. This type of problems is also referred to as ``best-arm identification" or ``pure exploration". In the second type, pulling an arm is always associated with a cost and constrained by a budget, no matter in the exploration phase or the exploitation phase, and the objective usually is to design an arm pulling algorithm in order to maximize the total reward with given cost or budget constraint.
References \cite{Tran-Thanh:AAAI:2010,Tran-Thanh:2012:AAAI,Burnetas:2012} consider the problem when the cost of pulling each arm is {\it fixed} and becomes known after the arm is used once. A sample-path cost constraint with {\it known} bandit dependent cost is considered in \cite{burnetas2017asymptotically}. References \cite{Ding:2013:budget,Xia:2015,Xia:2015:Thompson,Xia:2016:BudgetMP} study the budgeted bandit problems with {\it random} arm pulling costs. Reference \cite{Badanidiyuru:2013} considers the knapsack problem where there can be more than one budget constraints and shows how to construct polices with sub-linear regret.

In the proposed CCB model, the {\it net reward} function is related to the cost of pulling arms and the learning agent faces a ``soft constraint" on the cost instead of a fixed budget constraint. If the learner only pulls one arm in each step, the cost of pulling an arm can be absorbed into the reward of that arm (i.e., net reward). Our model then reduces to a conventional MAB model for this case. In this paper, however, we are interested in the scenario where the learner is allowed to sequentially pull a number of arms in each step, and the reward obtained in each step cannot be decomposed into the summation of the rewards from the pulled arms. Thus, the cost cannot be simply absorbed into the reward of individual arms. The intricate relationship between the cost and the reward, and its implications on the optimal policy require more sophisticated analysis, and that is the focus of this paper.

MAB with more than one arm to pull in each step has been studied in multiple-play MAB (MP-MAB) models in \cite{anantharam1987asymptotically,komiyama15,Jain:MPMAB:2014}, cascading bandits (CB) models in \cite{kveton2015cascading,kveton2015combinatorial,Zong:2016:CBL}, and ranked bandits (RB) models in \cite{RadlinskiRB2008,Streeter:NIPS:2008,Slivkins:2013:RBM}. Under MP-MAB, the learner is allowed to pull $L$ out of $K$ arms, and the reward equals to the summation of the rewards from individual arms. Although the proposed CCB model also allows the user to pull multiple arms in each step, the reward is not accumulative, thus leading to different solutions.

The CCB model proposed in this paper is closely related to the CB model investigated in~\cite{kveton2015cascading}. Specifically, \cite{kveton2015cascading} considers the scenario where at each step, $L$ out of $K$ items are listed by a learning agent and presented to a user. The user examines the ordered list from the first to the last, until he/she finds the first attractive item and clicks it. The system receives a reward if the user finds at least one of the items to be attractive. Our model has the same reward model as that in the cascading bandits setting. However, there are also important distinctions between them. In the CB model in \cite{kveton2015cascading}, the total number of items to be examined each step is fixed, and the cost of pulling individual arms is not considered. As a result, the same subset of items on a list will give the same expected reward, irrespective of their order on the list. However, for the proposed CCB model, the ranking of the items on a list does affect the expected net reward. Therefore, the structure of the optimal offline policy and the online algorithm we develop are fundamentally different from those in \cite{kveton2015cascading}.

The proposed CCB model is also related to RB in the sense that the order of the arms to pull in each step matters. A crucial feature of RB is that the click probability for a given item may depend on the item and its position on the list, as well as the items shown above. However, in our case, we assume the state of an arm evolves in an i.i.d. fashion from step to step.

\section{System Model and Problem Formulation}
Consider a $K$-armed stochastic bandit system where the state of each arm evolves independently from step to step. Let $X_{i,t}$ be the state of arm $i\in[K]$ in step $t$. Then, $X_{i,t} \in \{0,1\}$ evolves according to an i.i.d. Bernoulli distribution with parameter $\theta_i$. Denote $Y_{i,t}$ as the cost of pulling arm $i$ in step $t$, where $Y_{i,t} \in [0,1]$ evolves according to an i.i.d. unknown distribution with $\Eb[Y_{i,t}] = c_i$.

In step $t$, the learning agent chooses an ordered list of arms from $[K]$ and pull the arms sequentially. Once an arm is pulled, its state and the pulling cost are revealed instantly. Denote the ordered list as $I_t:= \left\{I_t(1), I_t(2), \ldots, I_t(|I_t|) \right\}$, where $I_t(i)$ is the $i${th} arm to be pulled, and $|I_t|$ is the cardinality of $I_t$.
Denote $\tilde{I}_t$ as the list of arms that have been actually pulled in step $t$. We have  $\tilde{I}_t\subseteq I_t$.

The {\it reward} that the learning agent receives in step $t$ depends on both $\tilde{I}_t$ and $\{X_{i,t}\}_{i\in \tilde{I}_t} $. Specifically, it can be expressed as $1-\prod_{i=1}^{|\tilde{I}_t|} (1-X_{\tilde{I}_t(i),t})$,
i.e., the learning agent gets reward one if one of the arms that have been examined in step $t$ has state 1; Otherwise, it equals zero.
The {\it cost} that is incurred at step $t$ also depends on $\tilde{I}_t$, and it can be expressed as
$\sum_{i=1}^{|\tilde{I}_t|} Y_{\tilde{I}_t(i),t}$.

With a given ordered list $I_t$,  $\tilde{I}_t$ is random and its realization depends on the observed $X_{i,t}$, $Y_{i,t}$ and the stopping condition in general. \ruida{Denote the {\it net reward} received by the learning agent at step $t$ as $r_t := 1-\prod_{i=1}^{|\tilde{I}_t|} (1-X_{\tilde{I}_t(i),t})-\sum_{i=1}^{|\tilde{I}_t|} Y_{\tilde{I}_t(i),t}.$} 
Define the  {\it per-step regret} $\textsf{reg}_t:= r_t^* -r_t$, \jing{where $r_t^*$ is the {\it net reward} that would be obtained at step $t$ if the statistics of $\{X_{i,t}\}_{i}$ and $\{Y_{i,t}\}_{i}$ were known beforehand and the optimal $I_t$ and stopping condition were adopted.}
Denote the observations up to step $t-1$ as $\Hc^{t-1}$. Then, $\Hc^{t-1}:=\cup_{\tau=1}^{t-1} \{X_{i,\tau},~Y_{i,\tau}\}_{i\in \tilde{I}_\tau}$. Without a priori statistics about $\{X_{i,t}\}$ and $\{Y_{i,t}\}$, our goal is to design an online algorithm to decide $I_{t}$ based on observations obtained in previous steps $\Hc^{t-1}$, and $\tilde{I}_t$ based on observed states and costs in step $t$,
so as to minimize the {\it cumulative regret} $R(T):= \Eb\left[\sum_{t=1}^T \ruida{\textsf{reg}_t}\right].$

In the following, we will first identify the structure of the optimal offline policy with a priori arm statistics, and then develop an online algorithm to learn the channel statistics and track the optimal offline policy progressively.

\section{Optimal Offline Policy}
\label{sec:offline}
We first study the optimal offline policy for the non-trivial case where $c_i>0$. We assume that the arm statistics $\{\theta_i\}_{i=1}^K$ and $\{c_i\}_{i=1}^K$ are known to the learning agent as prior knowledge. However, the instantaneous realization of rewards and costs associated with arms are unknown to the learning agent until they are pulled. Under the assumption that the distributions of arm states and costs are i.i.d. across steps, the optimal offline policy should remain the same for different steps. Thus, in this section, we drop the step index and focus on the policy at an individual step.

According to the definition of the cascading feedback, for any ordered list, the reward in each step will not grow after the learner observes an arm with state 1. Therefore, {\it to maximize the net reward, the learner should stop examining the rest of the list when a state 1 is observed, in order to save the cost of examination}. Let the ordered list under the optimal offline policy be $I^*$, then
\begin{align}
\ruida{\Eb[r^*]} &= \max_{I} \sum_{i=1}^{|I|} (\theta_{I(i)} - c_{I(i)}) \prod_{j=1}^{i-1}(1-\theta_{I(j)}). \label{eqn:netrew}
\end{align}
We note that the expected net reward structure is more complex than \congr{the} standard multi-armed bandits problem or \congr{the} standard cascading model, and the optimal offline policy is not straightforward. By observing \congr{\eqref{eqn:netrew}}, \congr{we note that} there \congr{are} both \congr{a} subtraction part $\theta_{I(i)} - c_{I(i)}$ and \congr{a} product part $\prod_{j=1}^{i-1}(1-\theta_{I(j)})$ inside each summation \congr{term}. On one hand we should choose \congr{large} $\theta_i$ to increase the value of $\theta_i - c_i$, but on the other hand we should not choose \congr{an} arm \congr{with a} big gap between $\theta_i$ and $c_i$, \congr{because} big $\theta_i$ contributes to \congr{smaller} $1 - \theta_i$. In \congr{the} extreme case where no cost is \congr{assigned to} arms, the optimal policy is \congr{to} pull all arms \congr{in \emph{any}} order, and the problem becomes trivial.

For simplicity of the analysis, we make the following assumptions:
\begin{Assumptions}
1) $\theta_i \neq c_i$, for all $ i \in [K]$. 2) There exists a constant $\epsilon>0$, such that $c_{i}>\epsilon$ for all $i \in [K]$.
\end{Assumptions}

Under Assumptions 1, we \congr{present the optimal offline policy in Theorem~\ref{thm:offline}, which is called Unit Cost Ranking with Threshold 1 (UCR-T1), as it ranks the expected reward normalized by the average cost and compares against threshold one.} 

\begin{Theorem} \label{thm:offline}
Arrange the arm indices such that
\begin{align*}
\frac{\theta_{1^*}}{c_{1^*}} \geq \frac{\theta_{2^*}}{c_{2^*}} \geq \ldots\geq \frac{\theta_{L^*}}{c_{L^*}}
>1 >\frac{\theta_{(L+1)^*}}{ c_{(L+1)^*} }\geq \ldots \geq \frac{\theta_{K^*}}{ c_{K^*} }.
\end{align*}
Then, $I^*=\{1^*, 2^*, \ldots, L^*\}$, and the corresponding optimal per-step reward is
$ \ruida{\Eb[r^*]} = \sum_{i=1}^L(\theta_{i^*} - c_{i^*})\prod_{j=1}^{i-1}(1-\theta_{j^*}).$
\end{Theorem}

The proof of Theorem~\ref{thm:offline} is provided in Appendix~\ref{app: thm1}. Theorem~\ref{thm:offline} indicates that ranking the arms in a descending order of $\frac{\theta_i}{c_i}$ and only including those with $\frac{\theta_i}{c_i}>1$ in $I^*$ achieves a balanced tradeoff between the subtraction $\theta_i - c_i$ and the product $\prod(1-\theta_i)$. This is an important observation which will also be useful in the online policy design.

\section{Online Policy}
With the optimal offline policy explicitly described in Theorem~\ref{thm:offline}, in this section, we will develop an online algorithm to maximize the cumulative expected net rewards without a priori knowledge of $\{\theta_{i}\}_{i=1}^K$ and $\{c_i\}_{i=1}^K$.

Unlike the previous work on MAB, the net reward structure in our setting is rather complex. One difficulty is that the learner has to rank $\theta_i/{c}_i$ and compare it with threshold $1$ for exploitation. A method to deal with this difficulty is using an UCB-type algorithm following the Optimism in Face of Uncertainty (OFU) principle~\cite{auer2002finite}. More specifically, we use an UCB-type \congr{indexing policy} to rank the arms. Though the upper confidence bound is a biased estimation of the statistics, it will converge to the expectation asymptotically. 

\subsection{Algorithm and Upper Bound} \label{sec:uc}
The cost-aware cascading UCB (CC-UCB) algorithm is \congr{described} in Algorithm~\ref{alg:online2}. The costs are assumed to be random but the learning agent has no knowledge of their distributions. We use ${N_{i,t}}$ to track the number of steps that arm $i$ has been pulled up to step $i$, and $\hat{\theta}_{i,t}$, $\hat{c}_{i,t}$ to denote the sample average of ${\theta}_{i}$ and ${c}_{i}$ at step $t$, respectively. The UCB padding term on the state and cost of arm $i$ at step $t$ is $u_{i,t}:=\sqrt{\frac{\alpha \log t}{N_{i,t}}}$, where $\alpha$ is a positive constant \ruida{no less than} 1.5.

\congr{CC-UCB} adopts the OFU principle to construct an upper bound of the ratio $\frac{\theta_i}{c_i}$. The main technical difficulty and correspondingly our novel contribution, however, lies in the theoretical analysis. This is because we have to deal with two types of regret: the regret caused by pulling ``bad" arms $(\theta_i < c_i)$; and that caused by pulling ``good" arms in a wrong order. \congr{To the authors' best knowledge, the latter component has not been addressed in the bandit literature before, and is technically challgening.} The \congr{overall} regret analysis of \congr{CC-UCB is thus} complicated and non-trivial.

\begin{algorithm}[t]
	\caption{Cost-aware Cascading UCB \congr{(CC-UCB)}  \\{ Input}: $\epsilon$, $\alpha$.}
	\begin{algorithmic}[1]
	\State {\bf Initialization}: Pull all arms in $[K]$ once, and observe their states and costs.
	\While {$t$}
      \For{$i=1:K$}
	 	\State $U_{i,t} = \hat{\theta}_{i,t} + u_{i,t}$;
		\State $L_{i,t} = \max(\hat{c}_{i,t} - u_{i,t}, \epsilon)$;
		\If {$U_{i,t}/L_{i,t}>1$}
			$i\rightarrow I_t$;
			\EndIf
      \EndFor
    \State Rank arms in $I_t$ in the descending order of $\frac{U_{i,t}}{L_{i,t}}$.
		\For{$i=1:|I_t|$}
		\State Pull arm $I_t(i)$ and observe $X_{I_t(i),t}, Y_{I_t(i),t}$;
		\State $i\rightarrow \tilde{I}_t$;
		\If {$X_{I_t(i),t}=1$} break;
		\EndIf
		\EndFor
	 \State Update $N_{i,t}$, $\hat{\theta}_{i,t}$, $\hat{c}_{i,t}$ for all $i \in \tilde{I}_t$.
      \State $t=t+1$;
    \EndWhile
	\end{algorithmic}\label{alg:online2}
\end{algorithm}

We have the following main result for the cumulative regret upper bound of Algorithm~\ref{alg:online2}.
\begin{Theorem} \label{thm:online}
Denote $\Delta_{i} := c_i - \theta_i$.
The accumulative regret under Algorithm~\ref{alg:online2} is upper bounded as follows:
\begin{align*}
R(T)\leq \sum_{i \in [K]\backslash I^*} c_i\frac{16\alpha \log T}{\Delta^{2}_{i}} + O(1),
\end{align*}
where $[K]\backslash I^*$ includes all items in $[K]$ except those in $I^*$.
\end{Theorem}
The proof of Theorem~\ref{thm:online} is deferred to Section~\ref{sec:upper_bound}.

{\bf Remark:} In Theorem~\ref{thm:online}, the upper bound depends on $\left(c_i - \theta_i\right)^2$, while conventional upper bounds for UCB algorithms usually depend on the gap between the ranking parameters. It is because the regret caused by pulling ``good" arms in a wrong order is bounded, as shown in Section~\ref{sec:upper_bound}; Thus, the regret is mainly due to pulling ``bad" arms, which is determined by arm $i$ itself $(U_{i,t} > L_{i,t})$ and not related to the $\theta/c$ gap between the best arm and arm $i$. When $c_i$s are {\it known} a priori, the upper bound can be reduced by a factor of 4.

\if{0}
\subsection{Special Case: Random Costs with Known Distributions} \label{sec:kc}
As a special case, we assume that the learning agent knows the cost statistics before invoking the online algorithm. With this additional knowledge, we can modify Algorithm~\ref{alg:online2} to eliminate the need of estimating $\hat{c}_{k,t}$ and computing $L_{k,t}$. Instead, we can directly use the true mean value $c_k$ of arm $k$ to replace $L_{k,t}$ in Algorithm~\ref{alg:online2}. This simplifies the algorithm and, as we can see from Theorem~\ref{thm:online_known}, also allows us to derive a slightly better regret upper bound.


\begin{Theorem} \label{thm:online_known}
The accumulative regret by using the modified Algorithm~\ref{alg:online2} with known $c_k$, $\forall k=1, \cdots, K$,  is $O(\log T)$. More specifically,
\begin{eqnarray*}
R(T)&\leq& \sum_{j \in [K]\backslash I^*} K\frac{4\alpha \log T}{\Delta_{j}^2}+O(1).
\end{eqnarray*}
\end{Theorem}
\fi

\subsection{Analysis of the Upper Bound}\label{sec:upper_bound}
Define $\Ec_t:=\{\exists i \in [K],~|\hat\theta_{i,t} - \theta_i|>u_{i,t}\mbox{ or }|\hat{c}_{i,t} - c_i|>u_{i,t}\}$, i.e. there exists an arm whose sample average of reward or cost lies outside the corresponding confidence interval. Denote $\bar\Ec_{t}$ as the complement of $\Ec_{t}$. Then, we have the following observations.

\begin{Lemma}\label{lemma:confidence}
If $\mathds{1}(\bar\Ec_{t})=1$, then, under Algorithm~\ref{alg:online2}, all arms in $I^*$ will be included in $I_t$.
\end{Lemma}
\begin{Proof}
According to Algorithm~\ref{alg:online2}, arm $i$ will be included in $I_t$ if $\frac{U_{i,t}}{L_{i,t}}\geq 1$. When $\mathds{1}(\Ec_t)=0$, we have $|\hat\theta_{i,t} - \theta_i|<u_{i,t}$, $|\hat{c}_{i,t} - c_i|<u_{i,t}$. Thus, $\hat\theta_{i,t} +u_{i,t}\geq  \theta_i$, $\hat{c}_{i,t}-u_{i,t}\leq \max\{\hat{c}_{i,t}-u_{i,t},\epsilon\}\leq  c_i$, which implies that $\frac{U_{i,t}}{L_{i,t}}\geq 1$.
\end{Proof}

\begin{Lemma}\label{lemma:e_t}
Under Algorithm~\ref{alg:online2}, we have
$\sum_{t=1}^T \Eb[\mathds{1}(\Ec_t)] \leq \psi:= K\left(1+\frac{\jing{4}\pi^2}{3}\right) $. 
\end{Lemma}

Then, define $\Bc_t:=\{\exists i^*,j^* \in I^*, \jing{i<j}, \mbox{ s.t. } \frac{U_{i^*,t}}{L_{i^*,t}}<\frac{U_{j^*,t}}{L_{j^*,t}}\}$, which represents the event that arms from $I^*$ are not ranked in the correct order. Since those arms are pulled linearly often in order to achieve small regret, intuitively, $\Bc_t$ happens with small probability.

\begin{Lemma}\label{lemma:B_t}
$\Eb\left[\sum_{t=1}^T\mathds{1}(\bar\Ec_{t}) \mathds{1}(\Bc_{t})\right] \leq \zeta =O(1)$.
\end{Lemma}

The proof of Lemma~\ref{lemma:e_t} is based on Hoeffding's inequality. The proof of Lemma~\ref{lemma:B_t} is derived based on the observation that the arms in $I^*$ are pulled linearly often if $\bar\Ec_t$ is true, and the corresponding confidence intervals $u_{i,t}$ shrinks fast in time. As $u_{i,t}$ decreases below certain threshold, $\bar\Bc_t$ happens. The detailed proofs of Lemma~\ref{lemma:e_t} and Lemma~\ref{lemma:B_t} can be found in Appendix~\ref{app: lem2} and Appendix~\ref{appx:lemma:B_t}, respectively.

\begin{Lemma}~\label{lemma:tighter}
Consider an ordered list $I_t$ that includes all arms from $I^*$ with the same relative order as in $I^*$. Then, under Algorithm~\ref{alg:online2}, 
$\ruida{\Eb[\textsf{reg}_t \mid \tilde{I}_t]} \leq \sum_{i \in \ruida{\tilde{I}_t} \backslash I^*} c_i$.
\end{Lemma}

\begin{Proof}
First, we point out the difference between $I_t$ and $I^*$ is that in $I_t$, there may \congr{exist} arms from $[K]\backslash I^*$ \congr{that are} inserted between the arms in $I^*$. Denoted such ordered subset of arms as $I_t\backslash I^*$. Then, depending on the realization of the states of the arm on $I_t$, a random subset of $I_t\backslash I^*$ will be pulled (i.e., $\tilde{I}_t\backslash I^*$), resulting in \congr{a} different regret. Denote the index of the last pulled arm in $\tilde{I}_t\backslash I^*$ as $\tilde{i}$. Then, depending on the state of arm $\tilde{i}$, there are two possible cases:

i) $X_{\tilde{i}, t}=0$. For this case, the regret comes from the cost of pulling the arms in $\tilde{I}_t\backslash I^*$ only. This is because if $I^*$ were the list of arms to pull, with the same realization of arm states, the learner would only pull the arms in $\tilde{I}_t\cap I^*$ and \congr{receive} the same reward. Thus, given $\tilde{I}_t$ and $X_{\tilde{i}, t}=0$, we have \jing{$\Eb[r^*_t - r_t \mid \tilde{I}_t,X_{\tilde{i}, t}=0] = \sum_{i \in \tilde{I}_t\backslash I^*} c_i$}.

ii) $X_{\tilde{i}, t}=1$. This indicates that $\tilde{i}$ is the last arm on $\tilde{I}_t$ due to the stopping condition. For this case, the learner spends costs on pulling the arms in $\tilde{I}_t\backslash I^*$ \congr{but} also receives \congr{the full reward one}. If $I^*$ were the list of arms to pull, with the same realization of arm states, the learner would first pull all arms in $\tilde{I}_t\cap I^*$. Since the states of such arms should be 0, \congr{she} would then continue pulling the \congr{remaining} arms \congr{in} $I^*$ if there is any. Denote the net reward obtained from the \congr{remaining pullings} as $r(I^*\backslash \tilde{I}_t)$. Then, $\jing{\Eb}[r(I^*\backslash \tilde{I}_t)]\leq 1$. Therefore, given $\tilde{I}_t$ and $X_{\tilde{i}, t}=1$, we have 
$\jing{\Eb[r_t^*-r_t | \tilde{I}_t,X_{\tilde{i}, t}=1]}=\jing{\Eb}[r(I^*\backslash \tilde{I}_t)] -\left(1-\sum_{i\in \tilde{I}_t\backslash I^* } c_i \right)\leq \sum_{i\in \tilde{I}_t\backslash I^* } c_i$.

Combining both cases, we have 
\ruida{$\Eb[\textsf{reg}_t \mid \tilde{I}_t] = \Eb[r_t^* - r_t \mid \tilde{I}_t] \leq \sum_{i \in \tilde{I}_t\backslash I^*} c_i$},
\congr{which completes the proof.}
\end{Proof}

Next, we consider the regret resulted from including arms outside $I^*$ \congr{in} the list $I_t$. We focus on the scenario when both $\Ec_t$ and $\Bc_t$ are false. Notice that  the condition of Lemma~\ref{lemma:tighter} is satisfied \congr{in this scenario}. Thus,
\begin{align}
& \Eb\left[\sum_{t=1}^T \mathds{1}(\bar\Ec_t) \mathds{1}(\bar \Bc_t) \ruida{\Eb[\textsf{reg}_t \mid \tilde{I_t}]} \right] \nonumber \\
&\leq \Eb\left[\sum_{t=1}^T \mathds{1}(\bar\Ec_t) \mathds{1}(\bar \Bc_t) \left(\sum_{i \in \ruida{\tilde{I}_t}\backslash I^*} c_i \right)\right]\label{eqn:TightUpper1}\\
&= \Eb\left[\sum_{i \in [K]\backslash I^*}\sum_{t=1}^T \mathds{1}(\bar\Ec_t) \mathds{1}(\bar \Bc_t)  \mathds{1}(i \in \ruida{\tilde{I}_t})c_i \right]\\
&\leq \Eb\left[\sum_{i \in [K]\backslash I^*}\sum_{t=1}^T \mathds{1}(\bar\Ec_t) \mathds{1} \left(\frac{ U_{i,t}}{L_{i,t}} > 1\right) \ruida{\mathds{1}(i \in \tilde{I}_t)}c_i \right] \\
&\leq \sum_{i \in [K]\backslash I^*} c_{i} \frac{16\alpha \log T}{\Delta^{2}_{i}},\label{eqn:TightUpper2}
\end{align}
where (\ref{eqn:TightUpper1}) is based on Lemma~\ref{lemma:tighter}; (\ref{eqn:TightUpper2}) is due to the fact that $\mathds{1}(\bar\Ec_t)\mathds{1}\left(\frac{ U_{i,t}}{L_{i,t}} > 1\right) $ is always less than or equal to $\mathds{1}\left( \theta_i + 2u_{i,t}> c_i - 2u_{i,t}\right)$, which is equivalent to $\mathds{1}\left(N_{i,t}<\frac{16\alpha \log t}{\Delta_i^2}\right)$.

Denote $\delta^*$ as the largest possible per-step regret, which is bounded by $\sum_{i\in [K]} c_i$, corresponding to the worst case scenario that the learner pulls all arms but does not receive reward one. Then, combining the results from above, we have
\begin{align}
R(T) 
&= \Eb\left[\sum_{t=1}^{T} [\mathds{1}(\Ec_t) + \mathds{1}(\bar\Ec_t)]\jing{\textsf{reg}_t } \right] \nonumber \\
&\leq \delta^* \sum_{t=1}^T \Eb\left[\mathds{1}(\Ec_t) +\mathds{1}(\bar\Ec_t)\mathds{1}(\Bc_t)\right]\nonumber \\
&\quad+ \Eb\left[\sum_{t=1}^T \mathds{1}(\bar\Ec_t)\mathds{1}(\bar\Bc_t) \Eb[\textsf{reg}_t \mid \tilde{I}_t]\right] \nonumber \\
&\leq \delta^*( \zeta + \psi) +\sum_{j \in [K]\backslash I^*} c_j  \frac{16\alpha \log T}{\Delta^{2}_{j}}, \nonumber
\end{align}
which proves Theorem~\ref{thm:online}.

\section{Lower Bound}
Before presenting the lower bound, we first define $\alpha$-consistent policies.
\begin{Definition}
\label{def:alphacnst}
Consider online policies that sequentially pull arms in $I_t$ until one arm with state 1 is observed. If $\Eb\left[\sum_{t=1}^T \mathds{1}(I_t \not = I^*)\right] = o(T^{\alpha})$ for any $\alpha \in (0,1)$, the policy is $\alpha$-consistent.
\end{Definition}

\begin{Lemma}\label{lemma:AnyLower}
For any ordered list $I_t$, the per-step regret in step $t$ is lower bounded by
$\ruida{\Eb[\textsf{reg}_t ] \geq \Eb\left[\sum_{i \in \tilde{I}_t\backslash I^*} (c_i - \theta_i)\right]}$.
\end{Lemma}

\begin{Proof}
Consider an ordered list $I^*_t:=I_t\cap I^*$, i.e., the sub-list of $I_t$ that only contains the arms from $I^*$ while keeping their relative order in $I_t$. \jing{Denote $r_t(I_t^*)$ as the reward obtained at step $t$ by pulling the arms in $I_t^*$ sequentially.} We have
\begin{align}
\Eb[\textsf{reg}_t] &= \Eb[r_t^* - r_t(I_t^*) + r_t(I_t^*) - r_t] \nonumber \\
& \geq \Eb[r_t(I^*_t) - r_t], \label{eqn:lowerBnd}
\end{align}
where inequality (\ref{eqn:lowerBnd}) follows from the fact that $I^*$ maximizes the expected reward in every step.

Similar to the proof of Lemma~\ref{lemma:tighter}, we denote the index of the last pulled arm in $\tilde{I}_t\backslash I^*$ as $\tilde{i}$. Then, given $\tilde{I}_t$ and $X_{\tilde{i}, t}=0$, we have $\jing{\Eb[r_t(I_t^*)-r_t \mid \tilde{I}_t,X_{\tilde{i}, t}=0 ]}=\sum_{i\in \tilde{I}_t\backslash I^* } c_i$.

If $X_{\tilde{i}, t}=1$, based on the assumption that all policy should stop pulling in a step if a state 1 is observed, we infer that the arms in $\tilde{I}_t\cap I^*$ should have state 0. If $I^*_t$ were the list of arms to pull, with the same realization of arm states, the learner would  continue pulling the remaining arms in $I^*_t$ if there is any. Denote the net reward obtained from the rest pulling as $r(I_t^*\backslash \tilde{I}_t)$. Then, due to the definition of $I^*$, we have $\jing{\Eb}[r(I^*_t\backslash \tilde{I}_t)]\geq 0$. Therefore, given $\tilde{I}_t$ and $X_{\tilde{i}, t}=1$, we have
$\jing{\Eb[r_t(I_t^*)-r_t \mid \tilde{I}_t,X_{\tilde{i}, t}=1]}=\Eb[r(I_t^*\backslash \tilde{I}_t)] -(1-\sum_{i\in \tilde{I}_t\backslash I^* } c_i )\geq \sum_{i\in \tilde{I}_t\backslash I^* } c_i -1$.

Combining both cases, we have
$\ruida{\Eb[r_t^*-r_t \mid \tilde{I}_t]} \geq \sum_{i\in \tilde{I}_t\backslash I^* } c_i-\theta_{\tilde{i}}\geq \sum_{i\in \tilde{I}_t\backslash I^* } (c_i-\theta_i)$.
Taking expectation with respect to $\tilde{I}_t$, we obtain the lower bound for $\jing{\Eb[\textsf{reg}_t]}$.
\end{Proof}

\begin{Theorem}
\label{thm:lowerbound}
Under any $\alpha$-consistent policy,
\begin{align*}
    \liminf_{T \rightarrow \infty} \frac{R(T)}{\log T} \geq \sum_{i \in [K]\backslash I^*}\frac{c_i - \theta_i}{d(\theta_i;c_i)}\quad \mbox{almost surely},
\end{align*}
where $d(\theta_i; c_i)$ is the KL divergence of Bernoulli distributions with means $\theta_i$ and $c_i$.
\end{Theorem}

\begin{proof}
According to Lemma~\ref{lemma:AnyLower}, we have
\begin{align}
\Eb[\ruida{\textsf{reg}_t}] &\geq \Eb\left[\sum_{i\in [K]\backslash I^* } \mathds{1}(i \in \tilde{I}_t)(c_i - \theta_i)\right].
\end{align}
Therefore,
\begin{align}
R(T) 
&\geq \Eb\left[\sum_{t=1}^T \left( \sum_{i \in [K]\backslash I^*} \mathds{1}(i \in \tilde{I}_t)(c_i - \theta_i)\right)\right]\\
&= \Eb\left[\sum_{i \in [K]\backslash I^*} (c_i - \theta_i)\left(\sum_{t=1}^T \mathds{1}(i \in \tilde{I}_t) \right) \right]\\
&= \sum_{i \in [K]\backslash I^*} (c_i - \theta_i)\Eb[N_{i,T}].\label{eqn:low3}
\end{align}
Leveraging the lower bound on $\Eb[N_{i,T}]$ from the proof of Theorem 4 in \cite{kveton2015cascading}, we have
\begin{equation}
\label{eqn:low4}
 \liminf_{T \rightarrow \infty}\frac{\Eb[N_{i,T}]}{\log T} \geq \frac{1}{d(\theta_i; c_i)}.
\end{equation}
Combining \eqref{eqn:low3} with \eqref{eqn:low4}, we obtain the lower bound.
\end{proof}

Comparing Theorem~\ref{thm:lowerbound} with Theorem~\ref{thm:online}, we conclude that Algorithm~\ref{alg:online2} achieves order-optimal regret performance.

\section{Experiments}
In this section we will resort to numerical experiments to evaluate the performances of the \congr{CC-UCB} algorithm in Algorithm~\ref{alg:online2}. We set $\alpha=1.5$ and $\epsilon=10^{-5}$. \congr{Both synthetic and real-world datasets are used.}

\if{0}
\subsection{Offline Policy}
We first examine the impact of cost on the structure of the optimal offline policy. We consider a 6-arm bandits setting with $\thetav=\{0.8,0.7,0.6,0.5,0.4, 0.3\}$. We consider a uniform cost for all of the arms, denoted as $c$. We gradually increase the value of $c$ from 0 to 1, and observe that the optimal offline policy has a threshold structure, as illustrated in Table~\ref{table:offline}. As we note, when the value of $c$ gradually increases, the number of arms that should be included in $I^*$ gradually decreases from 6 to 0. The transition happens when $c$ reaches one of the $\theta_i$s. The corresponding expected reward in each time also decreases simultaneously. The structure coincides with Theorem~\ref{thm:offline}, and indicates that cost-aware cascading bandits model is fundamentally different from its cost-oblivious counterpart~\cite{}. The cost of each arm essentially imposes a ``soft" constraint on the maximum number of arms to be pulled in each time, which can be adaptively controlled by varying the cost of arm pulling.
\begin{table}[h]
\centering
\begin{tabular}{l|ccc}
  \toprule[1.5pt]
  $c$ & $L$ & Reward  \\
  \midrule
  $[0, 0.3)$ & 6 & 0.9905\\
  $[0.3, 0.4)$ &  5 & 0.9928\\
  $[0.4, 0.5)$ & 4&0.9880  \\
  $[0.5, 0.6)$ & 3 &0.9760 \\
  $[0.6, 0.7)$ &2 &0.9400 \\
  $[0.7, 0.8)$& 1 &  0.8000 \\
   $[0.8, 1)$  & 0 & 0 \\
  \bottomrule[1.5pt]
\end{tabular}
\caption{Structure of the optimal offline policy.}\label{table:offline}
\end{table}
\fi

\subsection{Synthetic Data}
We consider a 6-arm bandits setting with $\thetav=\{0.8,0.7,0.6,0.5,0.4, 0.3\}$, and the mean of the cost $c=0.55$ for all arms. According to the UCR-T1 policy, we have $L=3$, i.e., the first three arms should be included in $I^*$. We then perform the CC-UCB algorithm under the assumption that both $\thetav$ and $c$ are unknown to the learning agent. We run it for $T=2\times 10^5$ steps, and average the accumulative regret over 20 runs. The result is plotted in Figure~\ref{fig:reg}(a). The error bar corresponds to one standard deviation of the regrets in 20 runs. We also study the case where the mean of the cost $c$ is known beforehand, however, the cost of each arm is still random. The result is also plotted in the same figure. As we observe, both curves increase sublinearly in $T$, which is consistent with the $O(\log T)$ bound we derive in Theorem~\ref{thm:online}. The regret with known cost statistics is significantly smaller than that of the unknown cost statistics case. 


Next, we evaluate the impact of system parameters on the performance of the algorithm. We vary $K$ and $L$, i.e., the total number of arms $K$, and the number of arms in $I^*$, respectively. We also change $\Delta_i$, i.e., $c_i-\theta_i$, for $i\in[K]\backslash I^*$. Specifically, we set $\theta_i=0.5$ for $i\in I^*$, $\theta_i=0.3$ for $i\in[K]\backslash I^*$, and let $c_i$ be a constant $c$ across all arms. By setting $\Delta_{L+1}$, the value of $c$ can be determined. The cumulative regrets at $T=10^5$  averaged over 20 runs are reported in Table~\ref{table:online}. We observe four major trends. First, the regret increases when the number of arms $K$ doubles. Second, the regret decreases when the number of arms in $I^*$ (i.e., $ L$) increases. Third, a prori knowledge of cost statistics always improve the regret. Fourth, when $c$ is known, the regret increases as $\Delta_{L+1}$ decreases. These trends are consistent with Theorem~\ref{thm:online}. However, the dependency on $\Delta_{L+1}$ when $c$ is unknown is not obvious from the experiment results. This might be because the algorithm depends on the estimation of the cost as well as the arm state, and the complicated interplay between cost and the optimal arm pulling policy makes the dependency on $\Delta_{L+1}$ hard to discern.

\begin{figure} [t]
	\centering
	\subfigure[Synthetic data]{ \includegraphics[width=1.625in]{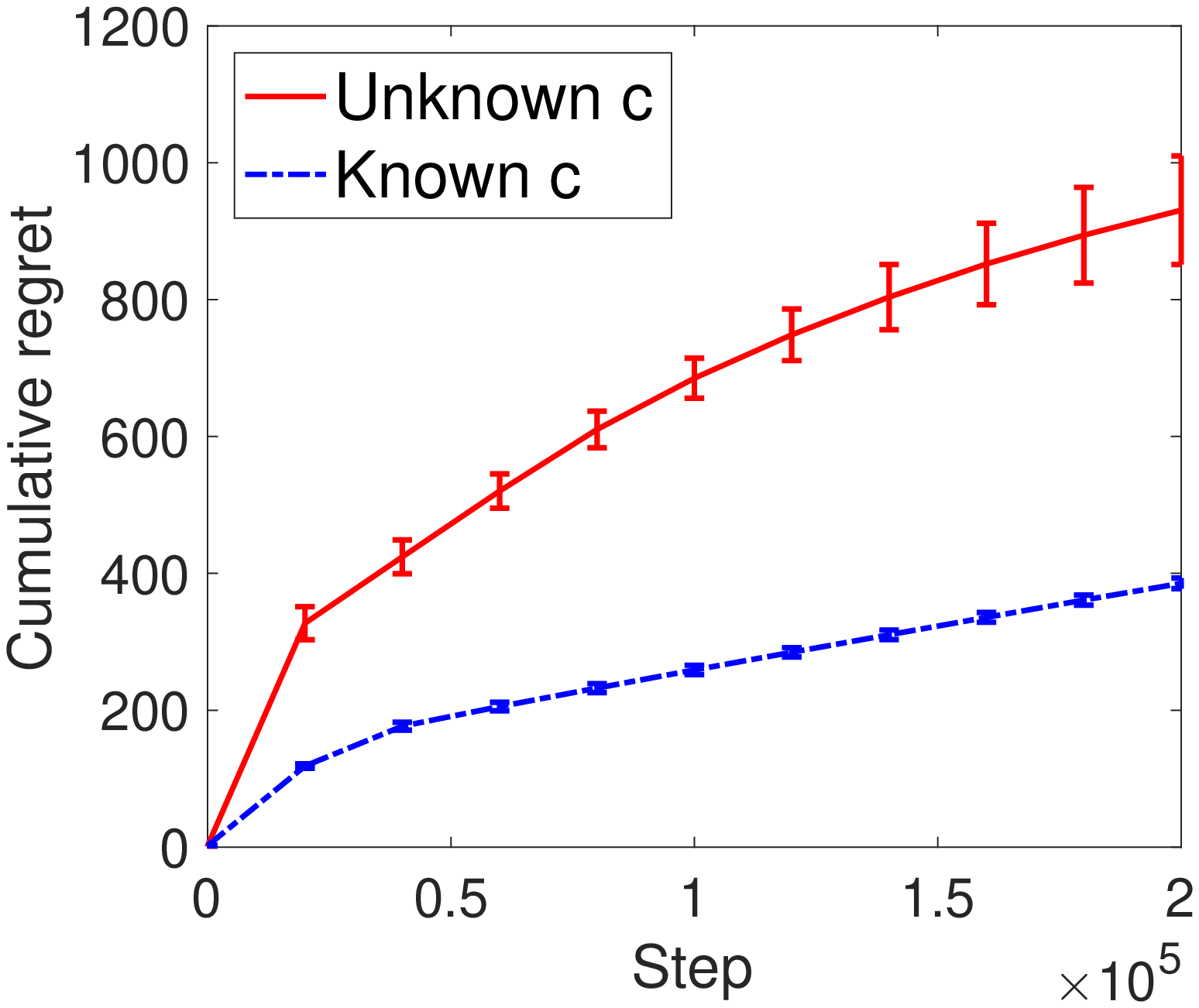}}
	\subfigure[Real-world data]{ \includegraphics[width=1.625in]{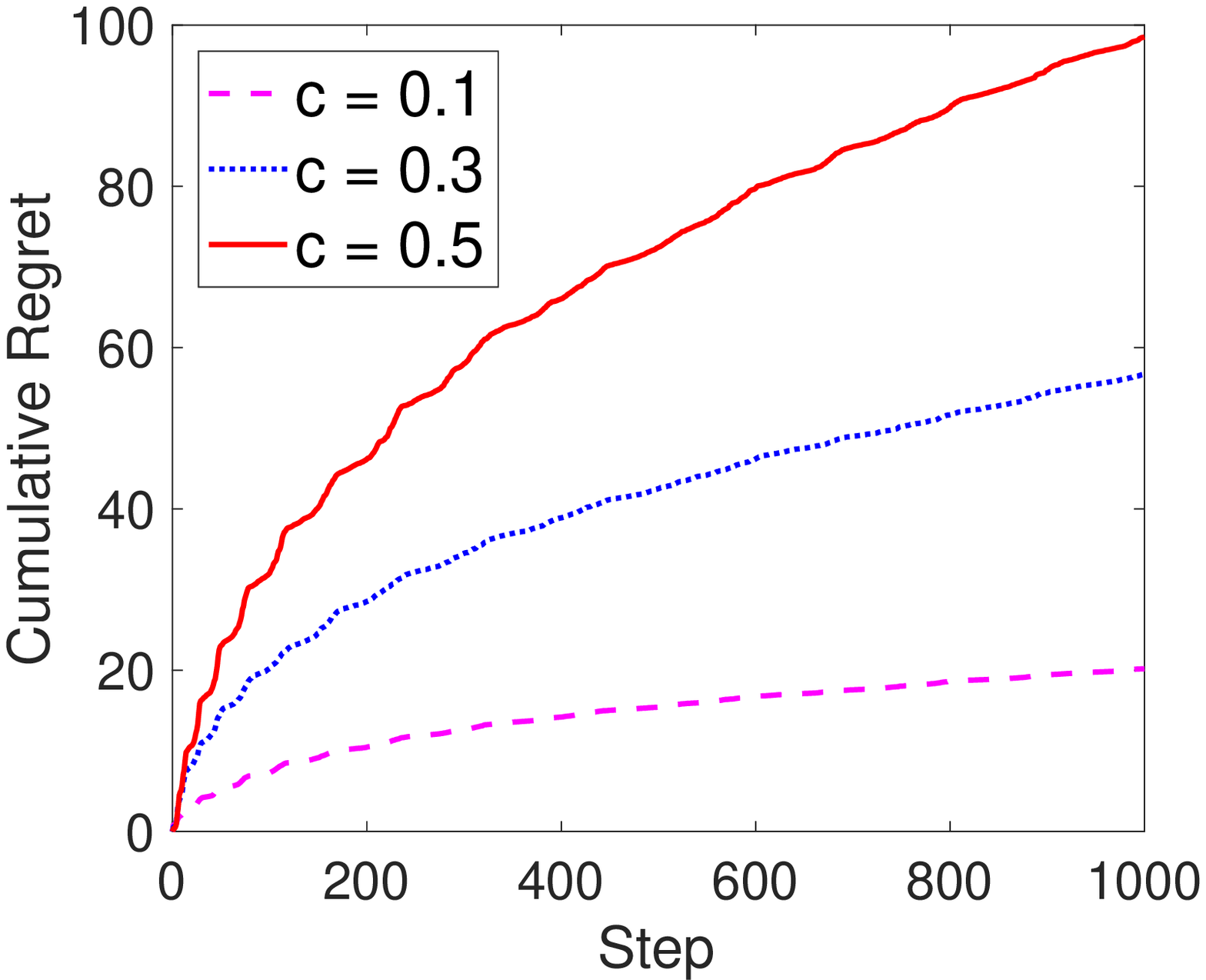}}
	\vspace{-0.1in}
	\caption{Cumulative regret versus step under CC-UCB.}
\label{fig:reg}
\end{figure}

\begin{table}[h]
\centering
\begin{tabular}{ccccc}
  \toprule[1.5pt]
  $K$ & $L$ & $\Delta_{L+1}$&  Known  $c$& Unknown $c$  \\
  \midrule
  $6$ & 1&  0.1& 580.3288& 2.2862e+03  \\
  $6$ &  3 & 0.1& 352.8772 & 1.4453e+03\\
  $6$ & 5 & 0.1&  117.5846 & 364.6771\\
  $12$ & 1 & 0.1&  2.5284e+03 & 1.0225e+04 \\
  $12$ &3 & 0.1& 1.2996e+03 & 4.8120e+03 \\
  $12$& 5 &0.1 &  387.8936  &  1.3728e+03\\
  $6$& 1 &  0.05 & 1.1536e+03 & 4.7941e+03\\
  $6$ &3 & 0.05 & 697.7550 & 1.4431e+03 \\
  $6$  & 5 & 0.05 &  160.7688& 212.0552 \\
  \bottomrule[1.5pt]
\end{tabular}
\caption{$T$-step regret in $T = 10^5$ steps.}\label{table:online}
\end{table}

\subsection{Real-world Data}
In this section, we test the proposed CC-UCB algorithm using real-world data extracted from the click log dataset of Yandex Challenge~\cite{InternetMathematics}. The click log contains complex user behaviors. To make it suitable for our algorithm, we extract data that contains click history of a specific query. We choose the total number of links for the query to be $15$ and set a constant and known cost $c$ for all of them. We estimate the probability that a user would click a link based on the dataset and use it as the ground truth $\theta_i$. We then test the CC-UCB based on the structure of the optimal offline policy. We plot the cumulative regret in Figure~\ref{fig:reg}(b) with different values of $c$. We observe that the cumulative regret grows sub-linearly in time, and monotonically increases as $c$ increases, which are consistent with Theorem~\ref{thm:online}. This indicates that the CC-UCB algorithm performs very well even when some of the assumptions (such as the i.i.d. evolution of arm states) we used to derive the performance bounds do not hold.

\section{Conclusions}
In this paper, we studied a CCB model by taking the cost of pulling arms into the cascading bandits framework. We first explicitly characterized the optimal offline policy UCR-T1, and then developed the CC-UCB algorithm for the online setting. We analyzed the regret behavior of the proposed CC-UCB algorithm by analyzing two different types of error events under the algorithm. We also derived an order-matching lower bound, thus proving that the proposed CC-UCB algorithm achieves order-optimal regret.  Experiments using both synthetic data and real-world data were carried out to evaluate the CC-UCB algorithm.

\bibliographystyle{named}

\appendix
\onecolumn
\section{Proof of Theorem 1}\label{app: thm1}
\if{0}
\setcounter{Theorem}{0}
\begin{Theorem}
Arrange the arm indices such that
\begin{align*}
\frac{\theta_{1^*}}{c_{1^*}} \geq \frac{\theta_{2^*}}{c_{2^*}} \geq \ldots\geq \frac{\theta_{L^*}}{c_{L^*}}
>1 >\frac{\theta_{(L+1)^*}}{ c_{(L+1)^*} }\geq \ldots \geq \frac{\theta_{K^*}}{ c_{K^*} }.
\end{align*}
Then, $I^*=\{1^*, 2^*, \ldots, L^*\}$, and the corresponding optimal per-step reward is
$ \ruida{\Eb[r^*]} = \sum_{i=1}^L(\theta_{i^*} - c_{i^*})\prod_{j=1}^{i-1}(1-\theta_{j^*}).$
\end{Theorem}
\fi

Let $I = (I(1),I(2),\ldots,I(|I|))$ be the list that \congr{is} presented by the player, where $|I| \leq K$. The expected net reward is
\begin{eqnarray}
\Eb[r(I)] = \sum_{i=1}^k( \theta_{I(i)} - c_{I(i)})\prod_{j=1}^{i-1}(1-\theta_{I(j)}).
\end{eqnarray}
If \congr{in the policy there} exists $i \in I$ satisfying $\frac{\theta_{I(i+1)}}{ c_{I(i+1)}} > \frac{\theta_{I(i)}}{c_{I(i)}}$ and $I(i),I(i+1) \in [K]$\congr{, then we can} define another list $I'$, which is \congr{only} different from policy $I$ by swapping \congr{the} $I(i)$ and $I(i+1)$ position\congr{:} $I' = (I(1),\ldots,I(i-1),I(i+1),I(i),I(i+2),\ldots,I(|I|))$, \congr{and} $|I'| = |I|$. Then the \congr{difference} between \congr{the} expected net rewards of $I$ and $I'$ is
\begin{align}
&\Eb[r(I')] - \Eb[r(I)] \nonumber\\
&= \left(\prod_{j=1}^{i-1} (1-\theta_{I(j)})\right){\big [}(\theta_{I(i+1)}-c_{I(i+1)})+(1- \theta_{I(i+1)})(\theta_{I(i)}-c_{I(i)}) - (\theta_{I(i)} -c_{I(i)}) - (1-\theta_{I(i)})(\theta_{I(i+1)} -c_{I(i+1)}){\big ]}  \nonumber \\
&=\left(\prod_{j=1}^{i-1}(1-\theta_{I(j)})\right) (c_{I(i)}\theta_{I(i+1)} - c_{I(i+1)}\theta_{I(i)}) \nonumber \\
&> 0, \label{eqn:swap1}
\end{align}
which implies that if \congr{the} presented arms \congr{from} $[K]$ \congr{are} not \congr{in a descending order in $\frac{\theta_i}{c_i}$}, then \congr{we can always create a new list that achieves better expected net reward by swapping positions of some arms.}

Besides, the reward is \congr{the} summation of $(\theta_{I{(i)}} - c_{I{(i)}})\prod_{j=1}^{i-1}(1-\theta_{I{(j)}})$. \congr{Then a term} will \congr{be positive} if $\theta_{I{(i)}} > c_{I{(i)}}$. \congr{As a result}, the optimal offline policy must contain all $i: \theta_i > c_i$. Combining with (\ref{eqn:swap1}) , we \congr{reach} the conclusion that the reward will \congr{be maximized by} presenting the top $L$ arms in a descending order based on $\frac{\theta_i}{c_i}$.

\section{Proof of Lemma 2}\label{app: lem2}
\if{0}
\setcounter{Lemma}{1}
\begin{Lemma}\label{lemma:e_t}
Under Algorithm~\ref{alg:online2}, we have
$\sum_{t=1}^T \Eb[\mathds{1}(\Ec_t)] \leq \psi:= K\left(1+\frac{\jing{4}\pi^2}{3}\right) $. 
\end{Lemma}
\fi
We note that
\begin{align}
&\sum_{t=1}^T \Eb[\mathds{1}(\Ec_t)] \nonumber \\
&\leq K +\sum_{t=K+1}^T \sum_{k \in [K]}\Big(\Pb\left[|\hat\theta_k(t) - \theta_k|>u_{k,t}\right]  +  \Pb\left[|\hat{c}_k(t) - c_k|>u_{k,t}\right]\Big) \\
&= K +\sum_{k \in [K]}\sum_{t=K+1}^T \sum_{n = 0}^{t} \Bigg(\Pb\left[|\hat{\theta}_k(t) - \theta_k|> \sqrt{\frac{\alpha \log t}{N_k(t)}}, N_k(t)=n \right]+\Pb\left[|\hat{c}_k(t) - c_k|> \sqrt{\frac{\alpha \log t}{N_k(t)}}, N_k(t)=n  \right] \Bigg) \nonumber \\
&\leq K + \sum_{k \in [K]}\sum_{t=K+1}^T \sum_{n = 1}^{t} 4\exp \left( -2 \frac{\alpha \log t}{n} n\right)	\label{eqn:hoeff21}\\
&= K +4 \sum_{k \in [K]} \sum_{t=K+1}^T t^{-2\alpha + 1} \leq K + K\frac{4\pi^2}{3} := \psi \label{eqn:confidence_bound}
\end{align}
where (\ref{eqn:hoeff21}) follows from the Hoeffding's inequality.

\section{Proof of Lemma~3}\label{appx:lemma:B_t}
Before we proceed to prove Lemma 3, we first introduce the following definitions.

Define a random variable $Z_{i,t}$ as follows:
  \begin{align}\label{defn:z_i}
Z_{i,t}&=\left\{ \begin{array}{ll}
 0, & \mbox{ if }\mathds{1}(\bar\Ec_{t})=0\\
 0, &\mbox{ if }\mathds{1}(\bar\Ec_{t})=1, ~\mbox{and}~\exists j \in [K]\backslash\{i\}, ~X_{j,t}=1\\ 
 1,  & \mbox{ if }\mathds{1}(\bar\Ec_{t})=1,  ~\mbox{and}~\forall j \in [K]\backslash\{i\}, ~X_{j,t}=0
 \end{array}\right..
 \end{align}
Denote $p_i := \frac{\prod_{j=1}^K(1-\theta_j)}{(1-\theta_i)}$. We can verify that
 \begin{align} \label{def:z_alt}
 Z_{i,t} & = \left\{ \begin{array}{ll}
 0, & \mbox{ if }\mathds{1}(\bar\Ec_{t})=0\\
 \mbox{Bernoulli } (p_i), &\mbox{ if }\mathds{1}(\bar\Ec_{t})=1
  \end{array}\right..
\end{align} 

As we will see later, we define $Z_{i,t}$ in such a way in order to lower bound the probability of the event ``$\Ec_t$ is false and arm $i$ is observed".

For any integer $n$, we define $\tau_n$ as the smallest step index such that $\sum_{t=1}^{\tau_n}\mathds{1}(\bar\Ec_t)=n$. This definition implies that $\mathds{1}(\bar\Ec_{\tau_n})=1$. Then, according to the definition of $Z_{i,t}$ in~(\ref{def:z_alt}), $Z_{i,\tau_n}$ is Bernoulli( $p_i$). Since $Z_{i, \tau_n}$ is $\sigma(\{X_{k,\tau_n}\}_{k\in[K]})$ measurable and $\{X_{k,t}\}_{k\in[K]}$ are independent, $\{Z_{i, \tau_n}\}_{n}$ are independent. Therefore, $\{Z_{t,\tau_n}\}_{n=1}^\infty$ are i.i.d. Bernoulli random variables with parameter $p_i$.

We then denote $\Gamma_T:=\sum_{t=1}^T\mathds{1}(\bar\Ec_t)$, i.e., the total number of steps up to $T$ when $\Ec_t$ is false. Then, we have the following observation.
\setcounter{Lemma}{5}
\begin{Lemma}\label{lemma:N_T_bound}
For all $i\in I^*$, $N_{i,T} \geq \sum_{t=1}^TZ_{i,t} =\sum_{n=1}^{\Gamma_T} Z_{i,\tau_n}.$
 \end{Lemma}

 \begin{Proof} 
 We note that
 \begin{align}
 N_{i,T} &= \sum_{t=1}^T \mathds{1}(i \in \tilde{I}_t) \\
 &\geq \sum_{t=1}^T \mathds{1}(\bar\Ec_t) \mathds{1}(i \in \tilde{I}_t) \\
 &\geq \sum_{t=1}^T \mathds{1}(\bar\Ec_t)\mathds{1}\left(\forall j \in [K] \backslash \{i\}, X_{j,t}=0 \right)\label{eqn:num_lower} \\
 &= \sum_{t=1}^T Z_{i,t} \label{eqn:def_z_i}
 \end{align}
 where (\ref{eqn:num_lower}) is based on the fact that arm $i$ will be pulled only when the states of all arms in $I_t$ listed before $i$ are $0$, and its probability is lower bounded by that of the extreme case when the states of all arms except $i$ are $0$. (\ref{eqn:def_z_i}) follows from the definition of $Z_{i,t}$ in (\ref{defn:z_i}). 
 
Since $\mathds{1}(\bar\Ec_{t})= \mathds{1} \left( t \in \bigcup_{n=1}^{\infty} \{ \tau_n\} \right)$, we have $\sum_{t=1}^TZ_{i,t} = \sum_{n=1}^{\Gamma_T}Z_{i,\tau_n}$, and Lemma~\ref{lemma:N_T_bound} follows.

 \end{Proof}
 
 Next, we are ready to prove Lemma~\ref{lemma:B_t}.
 
 \if{0}
 \setcounter{Lemma}{2}
 \begin{Lemma}\label{lemma:B_t}
$\Eb\left[\sum_{t=1}^T\mathds{1}(\bar\Ec_{t}) \mathds{1}(\Bc_{t})\right] \leq \zeta =O(1)$.
\end{Lemma}
\fi

Denote $\Delta_{i,j} := \frac{\left(\frac{\theta_i}{c_i} - \frac{\theta_j}{c_j}\right)c_j}{2\left(1+\frac{\theta_j}{c_j}\right)}$. Then, we have
\begin{align}
&\Eb\left[\sum_{t=1}^T\mathds{1}(\bar\Ec_{t}) \mathds{1}(\Bc_{t})\right] \nonumber \\
&\leq \sum_{j=2}^L \Eb \left[\sum_{t=1}^T \mathds{1}(\bar \Ec_t) \mathds{1}\left(\frac{U_{j^*,t}}{L_{j^*,t}} > \frac{\theta_{(j-1)^*}}{c_{(j-1)^*}}\right) \right]\\
&\leq \sum_{j=2}^L \Eb \left[\sum_{t=1}^T \mathds{1}(\bar \Ec_t) \mathds{1}\left(\frac{\theta_{j^*}+2u_{j^*,t}}{c_{j^*}-2u_{j^*,t}} > \frac{\theta_{(j-1)^*}}{c_{(j-1)^*}}\right) \right]\\
&= \sum_{j=2}^L\Eb\left[\sum_{t=1}^T  \mathds{1}(\bar \Ec_t) \mathds{1}\left(N_{j^*,t} < \frac{4 \left(1+\frac{\theta_{(j-1)^*}}{c_{(j-1)^*}}\right)^2\alpha \log t}{\left(\frac{\theta_{(j-1)^*}}{c_{(j-1)^*}} - \frac{\theta_{j^*}}{c_{j^*}}\right)^2 c_{j^*}^2} \right) \right]\\
&= \sum_{j=2}^L \Eb\left[\sum_{n=1}^{\Gamma_T} \mathds{1}\left(N_{j^*,\tau_n} < \frac{ \alpha \log \tau_n}{\Delta^{2}_{(j-1)^*,j^*}}\right) \right]\\
&= \sum_{j=2}^L \Eb\Bigg[\sum_{n=1}^{\Gamma_T} \mathds{1}\left(N_{j^*,\tau_n} < \frac{ \alpha \log \tau_n}{\Delta^{2}_{(j-1)^*,j^*}}\right) \Big(\mathds{1}\left(\tau_n \leq 2n\right) + \mathds{1}(\tau_n > 2n)\Big) \Bigg] \label{eqn:lemma3-0}
\end{align}

We note that 
\begin{align}
&\sum_{j=2}^L \Eb\Bigg[\sum_{n=1}^{\Gamma_T}  \mathds{1}\left(N_{j^*,\tau_n} < \frac{ \alpha \log \tau_n}{\Delta^{2}_{(j-1)^*,j^*}}\right) \mathds{1}\left(\tau_n \leq 2n\right) \Bigg] \nonumber \\
&\leq \sum_{j=2}^L \Eb\Bigg[\sum_{n=1}^{\Gamma_T} \mathds{1}\left(N_{j^*,\tau_n} < \frac{ \alpha (\log n + \log 2)}{\Delta^{2}_{(j-1)^*,j^*}}\right) \Bigg]\\
&\jing{\leq} \sum_{j=2}^L \Eb\Bigg[\sum_{n=1}^{\Gamma_T} \mathds{1}\left( \sum_{t=1}^{\tau_n}Z_{j^*,t} < \frac{ \alpha (\log n + \log 2)}{\Delta^{2}_{(j-1)^*,j^*}}\right)\Bigg] \label{eqn:lemma3-1} \\
&\ruida{\leq \sum_{j=2}^L \Eb\Bigg[\sum_{n=1}^{T} \mathds{1}\left( \sum_{t=1}^{n}Z_{j^*,\tau_t} < \frac{ \alpha (\log n + \log 2)}{\Delta^{2}_{(j-1)^*,j^*}}\right)\Bigg]} \label{eqn:lemma3-2} \\
&\ruida{=} \sum_{j=2}^L \sum_{n=1}^{\ruida{T}} \Pb \left( \sum_{t=1}^{n}Z_{j^*,\ruida{\tau_t}} - n p_{j^*}  < \frac{ \alpha (\log n + \log 2)}{\Delta^{2}_{(j-1)^*,j^*}} - n p_{j^*}\right)\\
&\ruida{\leq \sum_{j=2}^L\left( \zeta_j + \sum_{n=\zeta_j+1}^{\ruida{T}}\Pb \left( \sum_{t=1}^{n}Z_{j^*,\ruida{\tau_t}} - n p_{j^*}  < -\frac{p_{j^*}}{2} n \right) \right) \label{eqn:lemma3-3}}\\
&\leq \sum_{j=2}^L\left( \zeta_j + \sum_{n=\zeta_j+1}^{\ruida{T}}\exp\left(-2 \left(\ruida{\frac{p_{j^*}}{2}} \right)^2 n\right) \right) \label{eqn:lemma3-4}\\
&\leq \sum_{j=2}^L\left(  \zeta_j + \frac{2}{p_{j^*}^2}\right)\label{eqn:lemma3-4-2}
\end{align}
where (\ref{eqn:lemma3-1}) follows Lemma~\ref{lemma:N_T_bound}, (\ref{eqn:lemma3-2}) follows the fact that $\Gamma_T \leq T$. $\zeta_j$ in (\ref{eqn:lemma3-3}) is defined as the maximum $n$ such that  $\frac{ \alpha (\log n + \log 2)}{\Delta^{2}_{(j-1)^*,j^*}} \geq \frac{p_{j^*}}{2} n $. (\ref{eqn:lemma3-4}) follows from the fact that $Z_{j^*,\ruida{\tau_t}}$ are i.i.d. Bernoulli random variables with parameter $p_i$ and Hoeffding's inequality.

Besides,
\begin{align}
&\sum_{j=2}^L \Eb\Bigg[\sum_{n=1}^{\Gamma_T} \mathds{1}\left(N_{j^*,\tau_n} < \frac{ \alpha \log \tau_n}{\Delta^{2}_{(j-1)^*,j^*}}\right)\mathds{1}\left(\tau_n > 2n \right) \Bigg] \label{eqn:second_term} \\
&\leq \sum_{j=2}^L \Eb\Bigg[\sum_{n=1}^{\Gamma_T}\mathds{1}\left(\tau_n > 2n \right) \Bigg]  \\
&=\sum_{j=2}^L \Eb\Bigg[\sum_{n=1}^{\Gamma_T}\mathds{1}\left(\frac{\tau_n}{2} > \sum_{s=1}^{\tau_n}\mathds{1}(\bar\Ec_t) \right) \Bigg]\label{eqn:lemma3-5} \\
&\leq \sum_{j=2}^L \Eb\Bigg[\sum_{t=1}^{T}\mathds{1}\left(\frac{ t}{2} > \sum_{s=1}^t\mathds{1}(\bar\Ec_t)\right) \Bigg] \label{eqn:lemma3-6} \\
&= \sum_{j=2}^L \sum_{t=1}^{T}\Pb\Bigg[ \sum_{s=1}^t\mathds{1}(\Ec_s) >\frac{ t}{2} \Bigg] \label{eqn:lemma3-7}  
\end{align}
where (\ref{eqn:lemma3-5}) is based on the definition of $\tau_n$, (\ref{eqn:lemma3-6}) is due to the fact that $\Gamma_T \leq T$, thus $\sum_{n=1}^{\Gamma_T}\mathds{1}\left(\frac{\tau_n}{2} > \sum_{s=1}^{\tau_n}\mathds{1}(\bar\Ec_t) \right) \leq \sum_{t=1}^{T}\mathds{1}\left(\frac{ t}{2} > \sum_{s=1}^t\mathds{1}(\bar\Ec_t)\right) $.

We note that
\begin{align}
&\Eb[\mathds{1}(\Ec_t)] \leq\sum_{k \in [K]}\Big(\Pb\left[|\hat\theta_{k,t} - \theta_k|>u_{k,t}\right] +  \Pb\left[|\hat{c}_{k,t} - c_k|>u_{k,t}\right]\Big) \\
&= \sum_{k \in [K]}\sum_{n = 1}^{t} \Bigg(\Pb\left[|\hat{\theta}_{k,t} - \theta_k|> \sqrt{\frac{\alpha \log t}{N_{k,t}}}, N_{k,t}=n \right] +\Pb\left[|\hat{c}_{k,t} - c_k|> \sqrt{\frac{\alpha \log t}{N_{k,t}}}, N_{k,t}=n  \right] \Bigg)  \\
&\leq \sum_{k \in [K]} \sum_{n = 1}^{t} 4\exp \left( -2 \frac{\alpha \log t}{n} n\right)	\label{eqn:hoeff2}\\
&= 4 \sum_{k \in [K]}t^{-2\alpha + 1} = \frac{4K}{t^{2\alpha-1}} = \frac{4K}{t^2}. \label{eqn:littleProb}
\end{align}
Therefore,
\begin{align}
&\Pb\left[ \sum_{s=1}^t \mathds{1}(\Ec_s) >\frac{ t}{2} \right] \leq \frac{ \Eb\left[\left(\sum_{s=1}^{T}\mathds{1}(\Ec_s) \right)^2 \right]}{\left(t/2\right)^2} \label{eqn:chevbyshev} \\
&= \frac{4}{t^2} \left(\sum_{s=1}^t \Eb\left[\mathds{1}(\Ec_s)\right] + 2\sum_{1\leq i < j \leq t} \Eb \left[ \mathds{1}(\Ec_i)\mathds{1}(\Ec_j)\right] \right)  \\
&\leq \frac{4}{t^2} \left(\sum_{s=1}^t \Eb\left[\mathds{1}(\Ec_s)\right] + 2\sum_{1\leq i < j \leq t} \sqrt{\Eb\left[ \mathds{1}(\Ec_i)^2 \right]\Eb\left[ \mathds{1}(\Ec_j)^2 \right] } \right) \label{eqn:cauchy} \\
&\leq \frac{4}{t^2} \left(\sum_{s=1}^t\frac{4K}{s^2} + 2\sum_{1\leq i < j \leq t} \frac{4K}{ij}\right)\label{eqn:lemma3-8} \\
&= \frac{16K}{t^2} \left( \sum_{s=1}^t \frac{1}{s}\right)^2 < 16K \left( \frac{\log t + 1}{t} \right)^2 \label{eqn:as}
\end{align}
where (\ref{eqn:chevbyshev}) follows from Chebyshev's inequality, (\ref{eqn:cauchy}) follows from Cauchy's inequality, and (\ref{eqn:lemma3-8}) follows from (\ref{eqn:littleProb}).

Plugging (\ref{eqn:as}) into (\ref{eqn:lemma3-7}), we have
\begin{align}
(\ref{eqn:second_term})&\leq (L-1)\sum_{t=1}^{\infty}  16K \left( \frac{\log t + 1}{t} \right)^2 \nonumber \\
&< \jing{16}K\left(\frac{\pi^2}{6} + 1 + \log 2 + \frac{1}{3}(2 + \log^2 3 + 2\log 3) \right) :=\xi_0 \label{eqn:lin_bound}.
\end{align}

Plugging (\ref{eqn:lemma3-4-2}) and (\ref{eqn:lin_bound}) into (\ref{eqn:lemma3-0}), we have
\begin{align}
\Eb\left[\sum_{t=1}^T\mathds{1}(\bar\Ec_{t}) \mathds{1}(\Bc_{t})\right] \leq \sum_{j=2}^L\left(  \zeta_j + \frac{1}{2p_{j^*}^2}\right) + \xi_0 := \xi = O(1).
\end{align}

\end{document}